\documentclass[11pt,twoside,english]{article}
\usepackage{amsmath}
\usepackage{graphicx}
\usepackage{amssymb}
\usepackage{esint}
\usepackage[authoryear]{natbib}

\makeatletter

\usepackage{amsfonts}\usepackage{amsthm}
\usepackage{tabularx}\usepackage{dsfont}\usepackage{epsf}

\newcommand{\BEAS}{\begin{eqnarray*}}
\newcommand{\EEAS}{\end{eqnarray*}}
\newcommand{\BEA}{\begin{eqnarray}}
\newcommand{\EEA}{\end{eqnarray}}
\newcommand{\BEQ}{\begin{equation}}
\newcommand{\EEQ}{\end{equation}}
\newcommand{\BIT}{\begin{itemize}}
\newcommand{\EIT}{\end{itemize}}
\newcommand{\BNUM}{\begin{enumerate}}
\newcommand{\ENUM}{\end{enumerate}}

\newcommand{\BA}{\begin{array}}
\newcommand{\EA}{\end{array}}
\newcommand{\BC}{\begin{center}}
\newcommand{\EC}{\end{center}}

\newcommand{\eg}{{\it e.g.}}

\newcommand{\diag}{\mathop{\bf diag}}

\newtheorem{theorem}{Theorem}

\newtheorem{proposition}[theorem]{Proposition}
\newtheorem{remark}[theorem]{Remark}

\newcounter{exno}

{\begin{quote}}{\end{quote}}

\makeatletter
\long\def\@makecaption#1#2{
   \vskip 9pt
   \begin{small}
   \setbox\@tempboxa\hbox{{\bf #1:} #2}
   \ifdim \wd\@tempboxa > 5.5in
        \begin{center}
        \begin{minipage}[t]{5.5in}
        \addtolength{\baselineskip}{-0.95pt}
        {\bf #1:} #2 \par
        \addtolength{\baselineskip}{0.95pt}
        \end{minipage}
        \end{center}
   \else
    \hbox to\hsize{\hfil\box\@tempboxa\hfil}
   \fi
   \end{small}\par
}
\makeatother

\newcounter{oursection}

\newcounter{lecture}

\usepackage{dsfont}

\def\bX{\mathbf{X}}
\def\bY{\mathbf{Y}}

\def\bx{\mathbf{x}}

\def\bu{\mathbf{u}}

\def\bg{\mathbf{g}}

\newtheorem{definition}{Definition}

\newtheorem{lemma}{Lemma}

\usepackage{enumitem}
\usepackage{enumerate}
\usepackage{amsmath}

\def\E{\mathbb{E}}

\def\RR{\mathbb{R}}

\def\NN{\mathds{N}}

\def\undemi{\frac{1}{2}}

\def\Wcal{\mathcal{W}}

\def\Scal{\mathcal{S}}
\def\Xcal{\mathcal{X}}

\def\Rcal{\mathcal{R}}
\def\Ucal{\mathcal{U}}

\def\Mcal{M_+^b(\Xcal)}

\def\Ncal{\mathcal{N}}
\def\Hcal{\mathcal{H}}

\def\psd{\mbox{\bf S}_n^+}

\def\Kappa{\mathcal{K}}

\def\OMIT#1{}

\DeclareMathOperator{\spa}{span}

\DeclareMathOperator{\defi}{def}
\DeclareMathOperator{\tr}{tr}
 
\DeclareMathOperator{\defeq}{\overset{\defi}{=}}

\DeclareMathOperator{\var}{var}

\usepackage{enumerate,multirow}

\providecommand{\abs}[1]{\lvert#1\rvert} 
\providecommand{\norm}[1]{\lVert#1\rVert} 
\providecommand{\Norm}[1]{\left\lVert#1\right\rVert}

\usepackage{tabularx}

\makeatletter
\newif\if@borderstar
\def\bordermatrix{\@ifnextchar*{%
  \@borderstartrue\@bordermatrix@i}{\@borderstarfalse\@bordermatrix@i*}%
}
\def\@bordermatrix@i*{\@ifnextchar[{%
  \@bordermatrix@ii}{\@bordermatrix@ii[()]}
}
\def\@bordermatrix@ii[#1]#2{%
  \begingroup
    \m@th\@tempdima8.75\p@\setbox\z@\vbox{%
      \def\cr{\crcr\noalign{\kern 2\p@\global\let\cr\endline }}%
      \ialign {$##$\hfil\kern 2\p@\kern\@tempdima & \thinspace %
      \hfil $##$\hfil && \quad\hfil $##$\hfil\crcr\omit\strut %
      \hfil\crcr\noalign{\kern -\baselineskip}#2\crcr\omit %
      \strut\cr}}%
    \setbox\tw@\vbox{\unvcopy\z@\global\setbox\@ne\lastbox}%
    \setbox\tw@\hbox{\unhbox\@ne\unskip\global\setbox\@ne\lastbox}%
    \setbox\tw@\hbox{%
      $\kern\wd\@ne\kern -\@tempdima\left\@firstoftwo#1%
        \if@borderstar\kern2pt\else\kern -\wd\@ne\fi%
      \global\setbox\@ne\vbox{\box\@ne\if@borderstar\else\kern 2\p@\fi}%
      \vcenter{\if@borderstar\else\kern -\ht\@ne\fi%
        \unvbox\z@\kern-\if@borderstar2\fi\baselineskip}%
        \if@borderstar\kern-2\@tempdima\kern2\p@\else\,\fi\right\@secondoftwo#1 $%
    }\null \;\vbox{\kern\ht\@ne\box\tw@}%
  \endgroup
}
\makeatother

\def\ve{\varepsilon}
\def\bG{\textbf{G}}
\def\bx{\textbf{x}}

\def\ar{\text{ar}}

\def\bt{\mathbf{t}}

\def\bu{\mathbf{u}}

\def\bK{\mathbf{K}}

\def\E{\mathbb{E}}
\def\bu{\textbf{u}}

\def\RR{\mathbb{R}}

\def\NN{\mathds{N}}

\def\undemi{\frac{1}{2}}

\def\Scal{\mathcal{S}}
\def\Xcal{\mathcal{X}}

\def\Rcal{\mathcal{R}}
\def\Ucal{\mathcal{U}}

\def\Mcal{\mathcal{M}}

\def\Wcal{\mathcal{W}}

\def\Ncal{\mathcal{N}}
\def\Hcal{\mathcal{H}}

\def\Kappa{\mathcal{K}}

\def\kBOV{k_{\text{BoV}}^\kappa}
\def\kGA{k_\text{GA}^\kappa}
\def\kS{k_\text{S}}

\def\OMIT#1{}

\newtheorem{rem}{Remark}

\def\Psdd{\mbox{\bf S}_d^{++}}

\def\PsdN{\mbox{\bf S}_N^{+}}

\begin{document}

\title{Autoregressive Kernels for Time Series}

\author{Marco Cuturi\\Graduate School of Informatics\\Kyoto University\\\texttt{mcuturi@i.kyoto-u.ac.jp} \and
Arnaud Doucet\\Department of Computer Science \& Department of Statistics\\
 University of British Columbia\\  \texttt{arnaud@cs.ubc.ca}}
\maketitle

\begin{abstract}
We propose in this work a new family of kernels for variable-length
time series. Our work builds upon the vector autoregressive (VAR)
model for multivariate stochastic processes: given a multivariate
time series $\bx$, we consider the likelihood function $p_{\theta}(\bx)$
of different parameters $\theta$ in the VAR model as features to
describe $\bx$. To compare two time series $\bx$ and $\bx'$, we
form the product of their features $p_{\theta}(\bx)\cdot p_{\theta}(\bx')$
which is integrated out w.r.t $\theta$ {\normalsize using a }matrix
normal-inverse Wishart prior. Among other properties, this kernel
can be easily computed when the dimension $d$ of the time series
is much larger than the lengths of the considered time series $\bx$
and $\bx'$. It can also be generalized to time series taking values
in arbitrary state spaces, as long as the state space itself is endowed
with a kernel $\kappa$. In that case, the kernel between $\bx$ and
$\bx'$ is a a function of the Gram matrices produced by $\kappa$
on observations and subsequences of observations enumerated in $\bx$
and $\bx'$. We describe a computationally efficient implementation
of this generalization that uses low-rank matrix factorization techniques.
These kernels are compared to other known kernels using a set of benchmark
classification tasks carried out with support vector machines. 
\end{abstract}

\section{Introduction}

Kernel methods~\citep{hofmann2008kernel} have proved useful to handle
and analyze structured data. A non-exhaustive list of such data types
includes images~\citep{chapelle99SVMs,conf/iccv/GraumanD05,cuturi2006kernels,DBLP:conf/cvpr/HarchaouiB07},
graphs~\citep{KasTsuIno03,mahe2005,vishwanathan2008graph,shervashidze2009fast},
texts~\citep{joachims:2002a,moschitti2007fast} and strings on finite
alphabets~\citep{leslie02spectrum,cortes2004rational,VerSaiAku04,cuturi05context,Sonnenburg07largescale},
which have all drawn much attention in recent years. Time series,
although ubiquitous in science and engineering, have been comparatively
the subject of less research in the kernel literature.

Numerous similarity measures and distances for time series have been
proposed in the past decades~\citep{schreiber1997classification}.
These similarities are not, however, always well suited to the kernel
methods framework. First, most available similarity measures are not
positive definite~\citep{haasdonk-learning}. Likewise, most distances
are not negative definite. The positive definiteness of similarity
measures (alternatively the negative definiteness of distances) is
needed to use the convex optimization algorithms that underly most
kernel machines. Positive definiteness is also the cornerstone of
the reproducing kernel Hilbert space (RKHS) framework which supports
these techniques~\citep{berlinet03reproducing}. Second, most similarities
measures are only defined for `standard' multivariate time series,
that is time series of finite dimensional vectors. Yet, some of the
main application fields of kernel methods include bioinformatics,
natural language processing and computer vision, where the analysis
of time series of structured objects (images, texts, graphs) remains
a very promising field of study. Ideally, a useful kernel for time
series should be both positive definite and able to handle time series
of structured data. An oft-quoted example~\citep{bahlmann2002online,nips02-AA20}
of a non-positive definite similarity for time series is the Dynamic
Time Warping (DTW) score~\citep{Sakoe78}, arguably the most popular
similarity score for variable-length multivariate time series~\citep[\S4.7]{rabiner1993fundamentals}.
Hence the DTW can only be used in a kernel machine if it is altered
through ad hoc modifications such as diagonal regularizations~\citep{zhou2010unsupervised}.
Some extensions of the DTW score have addressed this issue:~\citet{hayashi2005embedding}
propose to embed time series in Euclidean spaces such that the distance
of such representations approximates the distance induced by the DTW
score.~\citet{cuturi07kernel} consider the soft-max of the alignment
scores of all possible alignments to compute a positive definite kernel
for two time series. This kernel can also be used on two time series
$\bx=(x_{1},\cdots,x_{n})$ and $\bx'=(x_{1}',\cdots,x_{n'}')$ of
structured objects since the kernel between $\bx$ and $\bx'$ can
be expressed as a function of the Gram matrix $\Kappa=\left[\kappa(x_{i},x'_{j})\right]_{i\leq n,j\leq n'}$
where $\kappa$ is a given kernel on the structured objects of interest.

A few alternative kernels have been proposed for multivariate time
series.~\citet{kumara2008large} consider a non-parametric approach
to interpolate time series using splines, and define directly kernels
on these interpolated representations. In a paragraph of their broad
work on probability product kernels,~\citet[\S4.5]{jebara04probability}
briefly mention the idea of using the Bhattacharyya distance on suitable
representations as normal densities of two time series using state-space
models.~\citet{vishwanathan2007binet} as well as \citet{borgwardt2006class}
use the family of Binet-Cauchy kernels~\citep{vishwanathan2004binet},
originally defined by the coefficients of the characteristic polynomial of kernel matrices
such as the matrix $\Kappa$ described above (when $n=n'$). Unlike other techniques listed above, these two proposals
rely on a probabilistic modeling of the time series to define a kernel.
Namely, in both the probability product and Binet-Cauchy approaches
the kernel value is the result of a two step computation: each time
series is first mapped onto a set of parameters that summarizes their
dynamic behavior; the kernel is then defined as a kernel between these
two sets of parameters.

The kernels we propose in this paper also rely on a probabilistic
modeling of time series to define kernels but do away with the two
step approach detailed above. This distinction is discussed later
in the paper in Remark~\ref{rem:comp}. Our contribution builds upon
the the \emph{covariance kernel} framework proposed by~\citet{seeger02covariance},
and whose approach can be traced back to the work of~\citet{haussler99convolution}
and~\citet{jaakkola99using}, who advocate the use of probabilistic
models to extract features from structured objects. Given a measurable
space $\Xcal$ and a model, that is a parameterized family of probability
distributions on $\Xcal$ of the form $\{p_{\theta},\theta\in\Theta\}$,
a kernel for two objects $\bx,\bx'\in\Xcal$ can be defined as 
\[
k(\bx,\bx')=\int_{\theta\in\Theta}p_{\theta}(\bx)\, p_{\theta}(\bx')\,\omega(d\theta),\]
where $\omega$ is a prior on the parameter space. Our work can be
broadly characterized as the implementation of this idea for time
series data, using the VAR model to define the space of densities
$p_{\theta}$ and a specific prior for $\omega$, the matrix-normal
inverse-Wishart prior. This conjugate prior allows us to obtain a
closed-form expression for the kernel which admits useful properties.

The rest of this paper is organized as follows. Section~\ref{sec:kernelcomput}
starts with a brief review of the Bayesian approach to dynamic linear
modeling for multivariate stochastic processes, which provides the
main tool to define autoregressive kernels on multivariate time series.
We follow by detailing a few of the appealing properties of autoregressive
kernels for multivariate time series, namely their infinite divisibility
and their ability to handle high-dimensional time series of short
length. We show in Section~\ref{sec:kernelization} that autoregressive
kernels can not only be used on multivariate time series but also
on time series taking values in any set endowed with a kernel. The
kernel is then computed as a function of the Gram matrices of subsets
of shorter time series found in $\bx$ and $\bx'$. This computation
requires itself the computation of large Gram matrices in addition
to a large number of operations that grows cubicly with the lengths
of $\bx$ and $\bx'$. We propose in Section~\ref{subsec:bounds}
to circumvent this computational burden by using low-rank matrix factorization
of these Gram matrices. We present in Section~\ref{sec:experiments}
different experimental results using toy and real-life datasets.


\section{Autoregressive Kernels}

\label{sec:kernelcomput} 
We introduce in this section the crux of our contribution, that is
a family of kernels that can handle variable-length multivariate time
series. Sections~\ref{subsec:auto}, \ref{subsec:bayeslinreg} and~\ref{subsec:bayesianVAR}
detail the construction of such kernels, while Sections~\ref{subsec:varandgram}
and~\ref{subsec:infdiv} highlight some of their properties.

\subsection{Autoregressive Kernels as an Instance of Covariance Kernels}

\label{subsec:auto} A vector autoregressive model of order $p$ henceforth
abbreviated as VAR($p$) is a family of densities for $\RR^{d}$-valued
stochastic processes parameterized by $p$ matrices $A_{i}$ of size
$d\times d$ and a positive definite matrix $V$. Given a parameter
$\theta=(A_{1},\cdots,A_{p},V)$ the conditional probability density
that an observed time series $\bx=(x_{1},x_{2},\dots,x_{n})$ has
been drawn from model $\theta$ given the $p$ first observations
$(x_{1},\dots,x_{p})$, assuming $p<n$, is equal to \[
p_{\theta}(\bx|x_{1},\cdots,x_{p})=\frac{1}{(2\pi|V|)^{\frac{d(n-p)}{2}}}\prod_{i=p+1}^{n}\exp\left(-\undemi\left\lVert \, x_{i}-\sum_{i=1}^{p}A_{i}x_{t-i}\,\right\rVert _{V}^{2}\right),\]
 where for a vector $x$ and a positive definite matrix $V$ the Mahalanobis
norm $\norm{x}_{V}^{2}$ is equal to $x^{T}V^{-1}x$. We write $\abs{\bx}$
for the length of the time series $\bx$, $n$ in the example above.
We abbreviate the conditional probability $p_{\theta}(\bx|x_{1},\cdots,x_{p})$
as $p_{\theta}(\bx)$ and take for granted that the $p$ first observations
of the time series are not taken into account to compute the probability
of $\bx$. We consider the set of parameters \[
\Theta=\underbrace{\RR^{d\times d}\times\cdots\times\RR^{d\times d}}_{p}\times\Psdd,\]
 where $\Psdd$ is the cone of positive definite matrices of size
$d\times d$, to define a kernel $k$ that computes a weighted sum
of the product of features of two time series $\bx$ and $\bx'$ as
$\theta$ varies in $\Theta$,

\begin{equation}
k(\bx,\bx')=\int_{\theta\in\Theta}p_{\theta}(\bx)^{c(\abs{\bx})}\, p_{\theta}(\bx')^{c(\abs{\bx'})}\omega(d\theta),\label{eq:renorm}\end{equation}
 where the exponential weight factor $c(\abs{\bx})$ is used to normalize
the probabilities $p_{\theta}(\bx)$ by the length of the considered
time series $\bx$.

\begin{rem} The feature map $\textup{\bx}\rightarrow\left\{ p_{\theta}(\textup{\bx})\right\} _{\theta\in\Theta}$
produces strictly positive features. Features will be naturally closer
to 0 for longer time series. We follow in this work the common practice
in the kernel methods literature, as well as the signal processing
literature, to normalize to some extent these features so that their
magnitude is independent of the size of the input object, namely the
length of the input time series. We propose to do so by normalizing
the probabilities $p_{\theta}(\textup{\bx})$ by the lengths of the
considered series. The weight $c(\abs{\textup{\bx}})$ introduced
in Equation~\ref{eq:renorm} is defined to this effect in section~\ref{subsec:bayesianVAR}.\end{rem}

\begin{rem}\label{rem:comp} The formulation of Equation~\eqref{eq:renorm}
is somehow orthogonal to kernels that map structured objects, in this
case time series, to a single density in a given model and compare
directly these densities using a kernel between densities, such as
probability product kernels~\citep{jebara04probability}, Binet-Cauchy
kernels~\citep{vishwanathan2007binet}, structural kernels~\citep{hein05hilbertian}
or information diffusion kernels~\citep{leb06metric}. Indeed, such
kernels rely first on the estimation of a parameter $\hat{\theta}_{\textup{\bx}}$
in a parameterized model class to summarize the properties of an object
$\textup{\bx}$, and then compare two objects $\textup{\bx}$ and
$\textup{\bx}'$ by using a proxy kernel on $\hat{\theta}_{\textup{\bx}}$
and $\hat{\theta}_{\textup{\bx}'}$. These approaches do require that
the model class, the VAR model in the context of this paper, properly
models the considered objects, namely that $\textup{\bx}$ is well
summarized by $\hat{\theta}_{\textup{\bx}}$. The kernels we propose
do not suffer from this key restriction: the VAR model considered
in this work is never used to infer likely parameters for a time series
$\textup{\bx}$ but is used instead to generate an infinite family
of features.\end{rem}

The kernel $k$ is mainly defined by the prior $\omega$ on the parameter
space. We present in the next section a possible choice for this prior,
the matrix-normal inverse-Wishart prior, which has been extensively
studied in the framework of Bayesian linear regression applied to
dynamic linear models.

\subsection{The Matrix-Normal Inverse-Wishart Prior in the Bayesian Linear Regression
Framework}

\label{subsec:bayeslinreg} Consider the regression model between
a vector of explanatory variables $x$ of $\RR^{m}$, and an output
variable $y\in\RR^{d}$, $y=Ax+\varepsilon,$ where $A$ is a $d\times m$
coefficient matrix and $\varepsilon$ is a centered Gaussian noise
with $d\times d$ covariance matrix $V$. Accordingly, $y$ follows
conditionally to $x$ a normal density \[
y|(x,A,V)\thicksim\Ncal(Ax,V).\]
 Given $n$ pairs of explanatory and response vectors $\left(x_{i},y_{i}\right)_{1\leq i\leq n}$
weighted by $n$ nonnegative coefficients $t=(t_{1},\ldots,t_{n})\geq0$
such that $\sum_{i=1}^{n}t_{i}=1$, the weighted likelihood of this
sample of $n$ observations is defined by the expression \[
\begin{aligned}\rho(Y|X,A,V,\Delta) & =\prod_{i=1}^{n}p(y_{j}|x_{j},A,V)^{t_{j}},\\
 & =\,\frac{1}{|2\pi V|^{1/2}}\exp\left(-\undemi\tr\Delta(Y-AX)^{T}V^{-1}(Y-AX)\right),\\
 & =\frac{1}{|2\pi V|^{1/2}}\exp\left(-\undemi\tr V^{-1}(Y-AX)\Delta(Y-AX)^{T}\right).\end{aligned}
\]
 where the matrices $\Delta,Y$ and $X$ stand respectively for $\diag(t_{1},\ldots,t_{n})\in\RR^{n\times n}$,
$[y_{1},\cdots,y_{n}]\in\RR^{d\times n}$ and $[x_{1},\ldots,x_{n}]\in\RR^{m\times n}$.

The matrix-normal inverse Wishart joint distribution~\citet[\S 16.4.3]{west1997bayesian}
is a natural choice to model the randomness for $(A,V)$. The prior
assumes that the $d\times m$ matrix $A$ is distributed following
a centered matrix-normal density with left variance matrix parameter
$V$ and right variance matrix parameter $\Omega$, \[
p(A)=\Mcal\Ncal(A;0,V,K)=\frac{1}{|2\pi V|^{m/2}|\Omega|^{d/2}}\exp\left(-\undemi\tr A^{T}V^{-1}A\Omega^{-1}\right)\]
 where $\Omega$ is a $m\times m$ positive definite matrix. Using
the following notations, \[
\BA{cc}S_{xx}=X\Delta X^{T}+\Omega^{-1},\quad S_{yx}=Y\Delta X^{T},S_{yy}=Y\Delta Y^{T},\quad S_{y|x}=S_{yy}-S_{yx}S_{xx}^{-1}S_{yx}^{T}.\EA\]
 we can integrate out $A$ in $\rho(Y|X,A,V,\Delta)$ to obtain \begin{equation}
\rho(Y|X,V)=\frac{1}{|\Omega|^{d/2}|S_{xx}|^{d/2}|2\pi V|^{1/2}}\exp\left(-\undemi\tr(V^{-1}S_{y|x})\right).\label{eq:yxv}\end{equation}
 The matrix-normal inverse Wishart prior for $(A,V)$ also assumes
that $V$ is distributed with inverse-Wishart density $\Wcal_{\lambda}^{-1}(\Sigma)$
of inverse scale matrix $\Sigma$ and degrees of freedom $\lambda>0$.
The posterior obtained by multiplying this prior by Equation~\eqref{eq:yxv}
is itself proportional to an inverse-Wishart density with parameters
$\Wcal^{-1}(\Sigma+S_{y|x},1+\lambda)$ which can be integrated to
obtain the marginal weighted likelihood, \[
\rho(Y|X)=\prod_{i=1}^{n}\frac{\Gamma(\frac{\lambda+2-i}{2})}{\Gamma(\frac{\lambda+1-i}{2})}\,\,\;\;\frac{1}{|\Omega|^{d/2}|S_{xx}|^{d/2}}\frac{|\Sigma|^{\lambda/2}}{|S_{y|x}+\Sigma|^{\frac{1+\lambda}{2}}}\]
 Using for $\Sigma$ the prior $I_{d}$, for $\Omega$ the matrix
$I_{m}$, and discarding all constants independent of $X$ and $Y$
yields the expression 
\[
\begin{aligned}\rho(Y|X)\propto\frac{1}{|X\Delta X^{T}+I_{m}|^{d/2}\,|Y(\Delta-\Delta X^{T}\left(X\Delta X^{T}+I_{m}\right)^{-1}X\Delta)Y^{T}+I_{d}|^{\frac{1+\lambda}{2}}},\end{aligned}
\]
Note that the matrix $H_{\Delta}\defeq\Delta X^{T}\left(X\Delta X^{T}+I_{m}\right)^{-1}X\Delta$
in the denominator is known as the hat-matrix of the orthogonal least-squares
regression of $Y$ versus $X$. The right term in the denominator
can be interpreted as the determinant of the weighted cross-covariance
matrix of $Y$ with the residues $(\Delta-H_{\Delta})Y$ regularized
by the identity matrix. 

\subsection{Bayesian Averaging over VAR Models}

\label{subsec:bayesianVAR} Given a VAR($p$) model, we represent
a time series $\bx=(x_{1},\cdots,x_{n})$ as a sample $X$ of $n-p$
pairs of explanatory variables in $\RR^{pd}$ and response variables
in $\RR^{d}$, namely $\{\left([x_{i},\cdots,x_{p+i-1}],x_{p+i}\right),\, i=1,\cdots,n-p\}$.
Following a standard practice in the study of VAR models~\citep[\S3]{lutkepohl2005nim},
this set is better summarized by matrices 
\[
X=\begin{bmatrix}\begin{bmatrix}\vdots\\
x_{1}\\
\vdots\end{bmatrix} & \cdots & \begin{bmatrix}\vdots\\
x_{n-p+1}\\
\vdots\end{bmatrix}\\
\vdots & \cdots & \vdots\\
\begin{bmatrix}\vdots\\
x_{p}\\
\vdots\end{bmatrix} & \cdots & \begin{bmatrix}\vdots\\
x_{n-1}\\
\vdots\end{bmatrix}\end{bmatrix}\in\RR^{pd\times n-p},\quad\text{ and }Y=\begin{bmatrix}\vdots & \cdots & \vdots\\
x_{p+1} & \cdots & x_{n}\\
\vdots & \cdots & \vdots\end{bmatrix}\in\RR^{d\times n-p}.\]

Analogously, we use the corresponding notations $X',Y'$ for a second
time series $\bx'$ of length $n'$. Using the notation \[
N\defeq n+n'-2p,\]
 these samples are aggregated in the $\RR^{N\times N}$, $\RR^{pd\times N}$
and $\RR^{d\times N}$ matrices \begin{equation}
\Delta=\diag\big(\frac{1}{2}\big[\underset{n-p\text{ times}}{\underbrace{\tfrac{1}{n-p},\cdots,\tfrac{1}{n-p}}},\underset{n'-p\text{ times}}{\underbrace{\tfrac{1}{n'-p},\cdots,\tfrac{1}{n'-p}}}\big]\big),\,\bX=\begin{bmatrix}X\, X'\end{bmatrix},\,\bY=\begin{bmatrix}Y\, Y'\end{bmatrix}.\label{eq:notations}\end{equation}

Note that by setting $A=[A_{1},\cdots,A_{p}]$ and $\theta=(A,V)$,
the integrand that appears in Equation~\eqref{eq:renorm} can be
cast as the following probability, \[
p_{\theta}(\bx)^{\frac{1}{2(n-p)}}\, p_{\theta}(\bx')^{\frac{1}{2(n'-p)}}=\rho(\bY|\bX,A,V,\Delta).\]
 Integrating out $\theta$ using the matrix-normal inverse Wishart
prior $\Mcal\Wcal_{\lambda}^{-1}(I_{d},I_{pd})$ for $(A,V)$ yields
the following definition:

\begin{definition} Given two time series $\bx,\bx'$ and using the
notations introduced in Equation~\eqref{eq:notations}, the autoregressive
kernel $k$ of order $p$ and degrees of freedom $\lambda$ is defined
as \begin{equation}
k(\textup{\bx,\bx'})=\frac{1}{|\bX\Delta\bX^{T}+I_{pd}|^{\frac{d}{2}}\,|\bY(\Delta-\Delta\bX^{T}\left(\bX\Delta\bX^{T}+I_{pd}\right)^{-1}\bX\Delta)\bY^{T}+I_{d}|^{\frac{1+\lambda}{2}}}.\label{eq:origkernel}\end{equation}
 \end{definition}


\subsection{Variance and Gram Based Formulations}

\label{subsec:varandgram} 
We show in this section that $k$ can be reformulated in terms of
Gram matrices of subsequences of $\bx$ and $\bx'$ rather than variance-covariance
matrices. For two square matrices $C\in\RR^{q\times q}$ and $D\in\RR^{r\times r}$
we write $C\backsim D$ when the spectrums of $C$ and $D$ coincide,
taking into account multiplicity, except for the value $0$. Recall
first the following trivial lemma. \begin{lemma}\label{lem:abba}
For two matrices $A$ and $B$ in $\RR^{q\times r}$ and $\RR^{r\times q}$
respectively, $AB\backsim BA$, and as a consequence $|AB+I_{q}|=|BA+I_{r}|$\end{lemma}

Based on this lemma, it is possible to establish the following result.

\begin{proposition}\label{prop:kernelgram}Let $\alpha\defeq\tfrac{1+\lambda}{d}$
then the autogressive kernel $k$ of order $p$ and degrees of freedom
$\lambda$ given in Equation~\eqref{eq:origkernel} is equal to \begin{equation}
k(\bx,\bx')=\left(\abs{\bX^{T}\bX\Delta+I_{N}}^{1-\alpha}\,\,\abs{\bX^{T}\bX\Delta+\bY^{T}\bY\Delta+I_{N}}^{\alpha}\right)^{-\frac{d}{2}},\label{eq:gramkernel}\end{equation}

\end{proposition}

\begin{proof}We use Lemma~\ref{lem:abba} to rewrite the first term
of the denominator of Equation~\eqref{eq:origkernel} using the Gram
matrix $\bX^{T}\bX$, \[
|\bX\Delta\bX^{T}+I_{pd}|=|\bX^{T}\bX\Delta+I_{N}|.\]
 Taking a closer look at the denominator, the matrix inversion lemma%
\footnote{$\left(A+UCV\right)^{-1}=A^{-1}-A^{-1}U\left(C^{-1}+VA^{-1}U\right)^{-1}VA^{-1}$%
} yields the equality \[
(\Delta-\Delta\bX^{T}\left(\bX\Delta\bX^{T}+I_{pd}\right)^{-1}\bX\Delta)=(\bX^{T}\bX+\Delta^{-1})^{-1}.\]
 Using again Lemma~\ref{lem:abba} the denominator of Equation~\eqref{eq:origkernel}
can be reformulated as \[
\left|\bY(\Delta-\Delta\bX^{T}\left(\bX\Delta\bX^{T}+I_{pd}\right)^{-1}\bX\Delta)\bY^{T}+I_{d}\right|=\frac{|\bX^{T}\bX+\Delta^{-1}+\bY^{T}\bY|}{|\bX^{T}\bX+\Delta^{-1}|}.\]

Factoring in these two results, we obtain Equation~\eqref{eq:gramkernel}.\end{proof}

We call Equation~\eqref{eq:origkernel} the \textbf{Variance} formulation
and Equation~\eqref{eq:gramkernel} the \textbf{Gram} formulation
of the autoregressive kernel $k$ as it only depends on the Gram matrices
of $\bX$ and $\bY$. Although both the Variance and Gram formulations
of $k$ are equal, their computational cost is different as detailed
in the remark below.

\begin{rem} In a \textbf{high $N$-low $d$ setting}, the computation
of $k$ requires $O(N(pd)^{2})$ operations to compute the denominator's
matrices and $O(p^{3}d^{3}+d^{3})$ to compute their inverse, which
yields an overall cost of the order of $O(Np^{2}d^{2}+(p^{3}+1)d^{3})$.
This may seem reasonable for applications where the cumulated time
series lengths' $N$ is much larger than the dimension $d$ of these
time series, such as speech signals or EEG data. In a \textbf{low
$N$-high $d$ setting}, which frequently appears in bioinformatics
or video processing applications, autoregressive kernels can be computed
using Equation~\eqref{eq:gramkernel} in $O((p+1)dN^{2}+N^{3})$
operations. \end{rem}

\subsection{Infinite Divisibility of Autoregressive Kernels}

\label{subsec:infdiv} We recall that a positive definite kernel function
$k$ is infinitely divisible if for all $n\in\NN$, $k^{1/n}$ is
also positive definite~\citep[\S3.2.6]{berg84harmonic}. We prove
in this section that under certain conditions on $\lambda$, the degrees-of-freedom
parameter of the inverse Wishart law, the autoregressive kernel is
infinitely divisible. This result builds upon~\citep[Proposition 3]{cuturi05semigroup}.

Proving the infinite divisibility of $k$ is useful for the following
two reasons: First, following a well-known result~\citet[\S3.2.7]{berg84harmonic},
the infinite divisibility of $k$ implies the negative definiteness
of $-\log k$. Using~\citet[\S3.3.2]{berg84harmonic} for instance,
there exists a mapping $\Phi$ of $\Xcal^{\NN}$ onto a Hilbert space
such that \[
\Norm{\Phi(\textup{\bx})-\Phi(\textup{\bx'})}^{2}=\frac{\log k(\textup{\bx,\bx})+\log k(\textup{\bx',\bx'})}{2}-\log k(\textup{\bx,\bx'}).\]
 and hence $k$ defines a Hilbertian metric for time series which
can be used with distance-based tools such as nearest-neighbors.

Second, on a more practical note, the exponent $d/2$ in Equation~\eqref{eq:gramkernel}
is numerically problematic when $d$ is large. In such a situation,
the kernel matrices produced by $k$ would be diagonally dominant.
This is analogous to selecting a bandwidth parameter $\sigma$ which
is too small when using the Gaussian kernel on high-dimensional data.
By proving the infinite divisibility of $k$, the exponent $d$ can
be removed and substituted by any arbitrary exponent.

To establish the infinite divisibility result, we need a few additional
notation. Let $\Mcal(\RR^{d})$ be the set of positive measures on
$\RR^{d}$ with finite second-moment $\mu[xx^{T}]\defeq\E_{\mu}[xx^{T}]\in\RR^{d\times d}$
. This set is a semigroup~\citep{berg84harmonic} when endowed with
the usual addition of measures.

\begin{lemma}\label{lem:igv} For two measures $\mu$ and $\mu'$
of $\Mcal(\RR^{d})$, the kernel \[
\tau:(\mu,\mu')\mapsto\frac{1}{\sqrt{\abs{(\mu+\mu')[xx^{T}]+I_{d}}}},\]
 is an infinitely divisible positive definite kernel. \end{lemma}
\begin{proof} The following identity is valid for any $d\times d$
positive-definite matrix $\Sigma$

\[
\frac{1}{\sqrt{\abs{\Sigma+I_{d}}}}=\frac{1}{\sqrt{(2\pi)^{d}}}\int_{\RR^{d}}e^{-\frac{1}{2}y^{T}(\Sigma+I_{d})\, y}dy\]
 Given a measure $\mu$ with finite second-moment on $\RR^{d}$, we
thus have \begin{equation}
\frac{1}{\sqrt{\abs{\mu[xx^{T}]+I_{d}}}}=\int_{\RR^{d}}e^{-\frac{1}{2}\langle\mu[xx^{T}]\,,\, yy^{T}\rangle}\; p_{\Ncal(0,I_{d})}(dy)\label{eq:semic}\end{equation}
 where $\langle\,,\,\rangle$ stands for the Frobenius dot-product
between matrices and $p_{\Ncal(0,I_{d})}$ is the standard multivariate
normal density. In the formalism of~\citep{berg84harmonic} the integral
of Equation~\eqref{eq:semic} is an integral of bounded semicharacters
on the semigroup $(\Mcal(\RR^{d}),+)$ equipped with autoinvolution.
Each semicharacter $\rho_{y}$ is indexed by a vector $y\in\RR^{d}$
as \[
\rho_{y}:\mu\mapsto e^{-\frac{1}{2}\langle\mu[xx^{T}]\,,\, yy^{T}\rangle}.\]
 To verify that $\rho_{y}$ is a semicharacter notice that $\rho_{y}(0)=1$,
where $0$ is the zero measure, and $\rho_{y}(\mu+\mu')=\rho_{y}(\mu)\,\rho_{y}(\mu')$
for two measures of $\Mcal(\RR^{d})$. Now, using the fact that the
multivariate normal density is a stable distribution, one has that
for any $t\in\NN$,

\[
\begin{aligned}\frac{1}{\sqrt{\abs{\mu[xx^{T}]+I_{d}}}} & =\int_{\RR^{d}}\left(e^{-\frac{1}{2t}\langle\mu[xx^{T}]\,,\, yy^{T}\rangle}\right)^{t}p_{\Ncal(0,I_{d}/t)}^{\otimes t}(dy),\\
 & =\left(\int_{\RR^{d}}e^{-\frac{1}{2t}\langle\mu[xx^{T}]\,,\, yy^{T}\rangle}\, p_{\Ncal(0,I_{d}/t)}(dy)\right)^{t},\end{aligned}
\]
 where $p^{\otimes t}$ is the $t$-th convolution of density $p$,
which proves the result. \end{proof}

\begin{theorem}For $0\leq\alpha\leq1$, equivalently for $0<\lambda\leq d-1$,
$\varphi\defeq-\frac{2}{d}\log k$ is a negative definite kernel\end{theorem}
\begin{proof}$k$ is infinitely divisible as the product of two infinitely
divisible kernels, $\tau^{d(1-\alpha)}$ and $\tau^{d\alpha}$ computed
on two different representations of $\bx$ and $\bx'$: first as empirical
measures on $\RR^{pd}$ with locations enumerated in the columns of
$X$ and $X'$ respectively, and as empirical measures on $\RR^{pd+d}$
with locations enumerated in the columns of the stacked matrices $[X;Y]$
and $[X';Y']$. The set of weights for both representations are the
uniform weights $\tfrac{1}{2(n-p)}$ and $\tfrac{1}{2(n'-p)}$. The
negative definiteness of $\varphi$ follows from, and is equivalent
to, the infinite divisibility of $k$~\citet[\S3.2.7]{berg84harmonic}.
\end{proof}

\begin{rem} The infinite divisibility of the joint distribution matrix
normal-inverse Wishart distribution would be a sufficient condition
to obtain directly the infinite divisibility of $k$ using for instance~\citet[\S3.3.7]{berg84harmonic}.
Unfortunately we have neither been able to prove this property nor
found it in the literature. The inverse Wishart distribution alone
is known to be not infinitely divisible in the general case~\citep{levy1948arithmetical}.
We do not know either whether $k$ can be proved to be infinitely
divisible when $\lambda>d-1$. The condition $0<\lambda\leq d-1$,
and hence $0\leq\alpha\leq1$ also plays an important role in Proposition~\ref{prop:convexity}.
\end{rem}

In the next sections, we will usually refer to the (negative definite)
autoregressive kernel \begin{equation}
\varphi(\bx,\bx')=C_{n,n'}+(1-\alpha)\log\,\abs{\bX^{T}\bX+\Delta^{-1}}+\alpha\log\,\abs{\bX^{T}\bX+\bY^{T}\bY+\Delta^{-1}}\label{eq:varphi}\end{equation}
 where the constant $C_{n,n'}=(n-p)\log(2(n-p))+(n'-p)\log(2(n'-p))$,
rather than considering $k$ itself.


\section{Extension to Time Series Valued in a Set Endowed with a Kernel}

\label{sec:kernelization} 
We show in Section~\label{subsec:kernelreform} that autoregressive
kernels can be extended quite simply to time series valued in arbitrary
spaces by considering Gram matrices. This extension is interesting
but can be very computationally expensive. We propose a way to mitigate
this computational cost by using low-rank matrix factorization techniques
in Section~\ref{subsec:bounds}

\subsection{Autoregressive Kernels Defined Through Arbitrary Gram Matrices}

Using again notation introduced in Equation~\eqref{eq:notations},
we write $K_{\bX}=\bX^{T}\bX$ for the $N\times N$ Gram matrix of
all explanatory variables contained in the joint sample $\bX$ and
$K_{\bY}=\bY^{T}\bY$ for the Gram matrix of all outputs of the local
regression formulations. As stated in Equation~\eqref{eq:varphi}
above, $\varphi$ and $k$ by extension can be defined as a function
of the Gram matrices $K_{\bX}$ and $K_{\bY}$. To alleviate notations,
we introduce two functions $g$ and $f$ defined respectively on the
cone of positive semidefinite matrices $\PsdN$ and on $(\PsdN)^{2}$:
\begin{equation}
g:\, Q\mapsto\log\,\abs{Q+\Delta^{-1}},\quad f:\,(Q,R)\mapsto(1-\alpha)g(Q)+\alpha g(R).\label{eq:definitionsFandG}\end{equation}
 Using these notations, we have 
\[
\varphi(\bx,\bx')=C_{n,n'}+f(K_{\bX},K_{\bX}+K_{\bY}),\]
which highlights the connection between $\varphi(\bx,\bx')$ and the
Gram matrices $K_{\bX}$ and $K_{\bX}+K_{\bY}$. In the context of
kernel methods, the natural question brought forward by this reformulation
is whether the linear dot-product matrices $K_{\bX}$ and $K_{\bY}$
in $\varphi$ or $k$ can be replaced by arbitrary kernel matrices
$\bK_{\bX}$ and $\bK_{\bY}$ between the vectors in $\RR^{pd}$ and
$\RR^{d}$ enumerated in $\bX$ and $\bY$, and the resulting quantity
still be a valid positive definite kernel between $\bx$ and $\bx'$.
More generally, suppose that $\bx$ and $\bx'$ are time series of
structured objects, graphs for instance. In such a case, can Equation~\eqref{eq:varphi}
be used to define a kernel between time series of graphs $\bx$ and
$\bx'$ by using directly Gram matrices that measure the similarities
between graphs observed in $\bx$ and $\bx'$? We prove here this
is possible.

Let us redefine and introduce some notations to establish this result.
Given a $k$-uple of points $\bu=(u_{1},\cdots,u_{k})$ taken in an
arbitrary set $\Ucal$, and a positive kernel $\kappa$ on $\Ucal\times\Ucal$
we write $\bK^{\kappa}(\bu)$ for the $k\times k$ Gram matrix \[
\bK^{\kappa}(\bu)\defeq\begin{bmatrix}\kappa(u_{i},u_{j})\end{bmatrix}_{1\leq i,j\leq k}.\]

For two lists $\bu$ and $\bu'$, we write $\bu\cdot\bu'$ for the
concatenation of $\bu$ and $\bu'$. Recall that an empirical measure
$\mu$ on a measurable set $\Xcal$ is a finite sum of weighted Dirac
masses, $\mu=\sum_{i=1}^{n}t_{i}\delta_{u_{i}}$, where the $u_{i}\in\Xcal$
are the locations and the $t_{i}\in\RR^{+}$ the weights of such masses.
\begin{lemma}\label{lem:igvkern} For two empirical measures $\mu$
and $\mu'$ defined on a set $\Xcal$ by locations $\bu=(u_{1},\cdots,u_{k})$
and $\bu'=(u'_{1},\cdots,u'_{l})$ and weights $\bt=(t_{1},\cdots,t_{k})\in\RR_{+}^{k}$
and $\bt'=(t'_{1},\dots,t'_{l})\in\RR_{+}^{l}$ respectively, the
function \[
\xi:(\mu,\mu')\mapsto\log\,\abs{\bK^{\kappa}(\bu\cdot\bu')\diag(\bt\cdot\bt')+I_{k+l}}\]
 is a negative definite kernel. \end{lemma} \begin{proof} We follow
the approach of the proof of~\citet[Theorem 7]{cuturi05semigroup}.
Consider $m$ measures $\mu_{1},\cdots,\mu_{m}$ and $m$ real weights
$c_{1},\cdots,c_{m}$ such that $\sum_{i=1}^{m}c_{i}=0$. We prove
that the quantity \begin{equation}
\sum_{i,j=1}^{m}c_{i}c_{j}\xi(\mu_{i},\mu_{j}),\label{eq:ndef}\end{equation}
 is necessarily non-positive. Consider the finite set $\Scal$ of
all locations in $\Xcal$ enumerated in all measures $\mu_{i}$. For
each point $u$ in $\Scal$, we consider the function $\kappa(u,\cdot)$
in the reproducing kernel Hilbert space $\Hcal$ of functions defined
by $\kappa$. Let $\Hcal_{\Scal}\defeq\spa\{\kappa(u),u\in\Scal\}$
be the finite dimensional subspace of $\Hcal$ spanned by all images
in $\Hcal$ of elements of $\Scal$ by this mapping. For each empirical
measures $\mu_{i}$ we consider its counterpart $\nu_{i}$, the empirical
measure in $\Hcal_{\Scal}$ with the same weights and locations defined
as the mapped locations of $\mu_{i}$ in $\Hcal_{\Scal}$. Since for
two points $u_{1},u_{2}$ in $\Scal$ we have by the reproducing property
that $\langle\kappa(u_{1},\cdot),\kappa(u_{2},\cdot)\rangle=\kappa(u_{1},u_{2})$,
we obtain that $\xi(\mu_{i},\mu_{j})=-\frac{1}{2}\log\tau(\nu_{i},\nu_{j})$
where $\tau$ is in this case cast as a positive definite kernel on
the Euclidean space $\Hcal_{\Scal}$. Hence the left hand side of
Equation~\eqref{eq:ndef} is nonnegative by negative definiteness
of $-\frac{1}{2}\log\tau$. \end{proof}

We now consider two time series $\bx$ and $\bx'$ taking values in
an arbitrary space $\Xcal$. For any sequence $\bx=(x_{1},\cdots,x_{n})$
we write $\bx_{i}^{j}$ where $1\leq i<j\leq n$ for the sequence
$(x_{i},x_{i+1},\cdots,x_{j})$. To summarize the transitions enumerated
in $\bx$ and $\bx'$ we consider the sequences of subsequences \[
X=(\bx_{1}^{p},\bx_{2}^{p+1},\cdots,\bx_{n-p+1}^{n-1}),\quad X'=({\bx'}_{1}^{p},{\bx'}_{2}^{p+1},\cdots,{\bx'}_{n'-p+1}^{n'-1}),\]
 and \[
Y=(x_{p+1},\cdots,x_{n}),\quad Y'=(x'_{p+1},\cdots,x'_{n'}).\]
 Considering now a p.d. kernel $\kappa_{1}$ on $\Xcal^{p}$ and $\kappa_{2}$
on $\Xcal$ we can build Gram matrices, \[
\bK_{1}=\bK^{\kappa_{1}}(X\cdot X'),\quad\bK_{2}=\bK^{\kappa_{2}}(Y\cdot Y').\]

\begin{theorem}\label{theo:kARk} Given two time series $\bx,\bx'$
in $\Xcal^{\NN}$, the autoregressive negative definite kernel $\varphi_{\kappa}$
of order $p$, parameter $0<\alpha\leq1$ and base kernels $\kappa_{1}$
and $\kappa_{2}$ defined as 
\[
\varphi_{\kappa}(\textup{\bx,\bx'})=C_{n,n'}+f(\bK_{1},\bK_{1}+\bK_{2}),\]
is negative definite. \end{theorem}

\begin{proof} $\varphi_{\kappa}$ is negative definite as the sum
of three negative definite kernels: $C_{n,n'}$ is an additive function
in $n$ and $n'$ and is thus trivially negative definite. The term
$(1-\alpha)\log\,\abs{K_{\bX}+\Delta^{-1}}$ can be cast as
$(1-\alpha)$ times the negative definite kernel $\xi$ defined on
measures of $\Xcal^{p}$ with kernel $\kappa_{1}$ while the term
$\alpha\log\,\abs{K_{\bX}+K_{\bY}+\Delta^{-1}}$ is $\alpha$
times the negative definite kernel $\xi$ defined on measures of $\Xcal^{p}\times\Xcal$
with kernel $\kappa_{1}+\kappa_{2}$. \end{proof}


\subsection{Approximations Using Low-Rank Factorizations}

\label{subsec:bounds} 
We consider in this section matrix factorization techniques to approximate
the kernel matrices $\bK_{1}$ and $\bK_{1}+\bK_{2}$ used to compute
$\varphi_{\kappa}(\bx,\bx')$ by low rank matrices. Theorem~\ref{theo:bounds}
provides a useful tool to control the tradeoff between the accuracy
and the computational speed of this approximation.

\subsubsection{Computing $f$ using low-rank matrices}

Consider an $N\times m_{1}$ matrix $\bg_{1}$ and an $N\times m_{2}$
matrix $\bg_{2}$ such that $\bG_{1}\defeq\bg_{1}\bg_{1}^{T}$ approximates
$\bK_{1}$ and $\bG_{2}\defeq\bg_{2}\bg_{2}^{T}$ approximates $\bK_{1}+\bK_{2}$.
Namely, such that the Frobenius norms of the differences \[
\ve_{1}\defeq\bK_{1}-\bG_{1},\quad\ve_{2}\defeq\bK_{1}+\bK_{2}-\bG_{2},\]
 are small, where the Frobenius norm of a matrix $M$ is $\norm{M}\defeq\sqrt{\tr M^{T}M}$.

Computing $f(\bG_{1},\bG_{2})$ requires an order of $O(N(m_{1}+m_{2})^{2}+m_{1}^{3}+m_{2}^{3})$
operations. Techniques to obtain such matrices $\bg_{1}$ and $\bg_{2}$
range from standard truncated eigenvalue decompositions, such as the
power method, to incomplete Cholesky decompositions~\citep{fine2002efficient,bach2005icml}
and Nyström methods~\citep{Williams01usingthe,Drineas2005} which
are arguably the most popular in the kernel methods literature. The
analysis we propose below is valid for any factorization method.

\begin{proposition}\label{prop:convexity} Let $0\leq\alpha\leq1$,
then $f$ defined in Equation \eqref{eq:definitionsFandG} is a strictly
concave function of\, $(\PsdN)^{2}$ which is strictly increasing
in the sense that $f(Q_{1},R_{1})<f(Q_{2},R_{2})$ if $Q_{2}\succ Q_{1}$
and $R_{2}\succ R_{1}$. \end{proposition}

\begin{proof} The gradient of $g:\, Q\mapsto\log\,\abs{Q+\Delta^{-1}}$
is $\nabla g(Q)=(Q+\Delta^{-1})^{-1}$ which is thus a positive definite
matrix. As a consequence, $f$ is a strictly increasing function.
The Jacobian of this gradient evaluated at $Q$ is the linear map
$\ve\in\psd\mapsto-\tr(Q+\Delta{-1})^{-1}\ve(Q+\Delta^{-1})^{-1}$.
For any matrix $C\succ0$ the Hessian of $g$ computed at $Q$ is
thus the quadratic form \[
\nabla^{2}g(Q):(\ve,\nu)\rightarrow-\tr(Q+\Delta^{-1})^{-1}\ve(Q+\Delta^{-1})^{-1}\nu\]
 Since $\tr UVUV=\tr((\sqrt{U}V\sqrt{U})^{2})>0$ for any two matrices
$U,V\succ0$, $\nabla^{2}g(Q)(\ve,\ve)$ is negative for any positive
definite matrix $\ve$. Hence the Hessian of $f$ is minus a positive
definite quadratic form on $(\PsdN)^{2}$ and thus $f$ is strictly
concave. \end{proof}

We use a first order argument to bound the difference between the
approximation and the true value of $f(\bK_{1},\bK_{1}+\bK_{2})$
using terms in $\norm{\ve_{1}}$ and $\norm{\ve_{2}}$:

\begin{multline*}
f(\bK_{1},\bK_{1}+\bK_{2})-(1-\alpha)\langle\nabla g(\bG_{1}),\ve_{1}\rangle-\alpha\langle\nabla g(\bG_{2}),\ve_{2}\rangle\\
\leq f(\bG_{1},\bG_{2})\leq f(\bK_{1},\bK_{1}+\bK_{2}).\end{multline*}

\begin{theorem}\label{theo:bounds} Given two time series $\textup{\bx,\bx'}$,
for any low rank approximations $\bG_{1}$ and $\bG_{2}$ in $\PsdN$
such that $\bG_{1}\preceq\bK_{1}$ and $\bG_{2}\preceq\bK_{1}+\bK_{2}$
we have that \[
e^{-\varphi_{\kappa}(\textup{\bx,\bx'})}\leq e^{-C_{n,n'}-f(\bG_{1},\bG_{2})}\leq(1+\rho)e^{-\varphi_{\kappa}(\textup{\bx,\bx'})},\]
 where $\rho\defeq\exp\left((1-\alpha)\norm{\nabla g(\bG_{1})}\norm{\ve_{1}}+\alpha\norm{\nabla g(\bG_{2})}\norm{\ve_{2}}\right)-1$.

\end{theorem} \emph{Proof.} Immediate given that $f$ is concave
and increasing $\blacksquare$.

\subsubsection{Early stopping criterion}

Incomplete Cholesky decomposition and Nyström methods can build iteratively
a series of matrices $\bg_{1,t}$ and $\bg_{2,t}\in\RR^{N\times t},1\leq t\leq N$
such that $\bG_{1,t}\defeq\bg_{1,t}\bg_{1,t}^{T}$ and $\bG_{2,t}\defeq\bg_{2,t}\bg_{2,t}^{T}$
increase respectively towards $\bK_{1}$ and $\bK_{1}+\bK_{2}$ as
$t$ goes to $N$. The series $\bg_{1,t}$ and $\bg_{2,t}$ can be
obtained without having to compute explicitly the whole of $\bK_{1}$
nor $\bK_{1}+\bK_{2}$ except for their diagonal.

The iterative computations of $\bG_{1,t}$ and $\bG_{2,t}$ can be
halted whenever an upper bound for each of the norms $\norm{\ve_{1,t}}$
and $\norm{\ve_{2,t}}$ of the residues $\ve_{1,t}\defeq\bK_{1}-\bG_{1,t}$
and $\ve_{2,t}\defeq\bK_{1}+\bK_{2}-\bG_{2,t}$ goes below an approximation
threshold.

Theorem~\ref{theo:bounds} can be used to produce such a stopping
criterion by a rule which combines an upper bound on $\norm{\ve_{1,t}}$
and $\norm{\ve_{2,t}}$ and the exact norm of the gradients of $g$
at $\bG_{1,t}$ and $\bG_{2,t}$. This would require computing Frobenius
norms of the matrices $(\bG_{i,t}+\Delta^{-1})^{-1}$, $i=1,2$. These
matrices can be updated iteratively using rank-one updates. A simpler
alternative which we consider is to bound $\norm{\nabla g(\bG_{i,t})}$
uniformly between $0$ and $\bK$ using the inequality \[
(\bG_{i,t}+\Delta^{-1})^{-1}\preceq\Delta,\, i=1,2.\]
 which yields the following bound: 
\[
e^{-\varphi_{\kappa}(\bx,\bx')}\leq e^{-C_{n,n'}-f(\bG_{1},\bG_{2})}\leq e^{-\varphi_{\kappa}(\bx,\bx')}e^{\tfrac{1}{2}\sqrt{\frac{N}{(n-p)(n'-p)}}\left((1-\alpha)\norm{\ve_{1}}+\alpha\norm{\ve_{2}}\right)}\]
We consider in the experimental section the positive definite kernel
$e^{-t\varphi_{\kappa}}$, that is the scaled exponentiation of $\varphi_{\kappa}$
multiplied by a bandwidth parameter $t>0$. Setting a target tolerance
$\sigma>0$ on the ratio between the approximation of $e^{-t\varphi_{\kappa}}$
and its true value, namely requiring that

\[
e^{-t\varphi_{\kappa}(\bx,\bx'))}\leq e^{-t\left(C_{n,n'}+f(\bG_{1},\bG_{2})\right)}\leq(1+\sigma)e^{-t\varphi_{\kappa}(\bx,\bx'))},\]
can be ensured by stopping the factorizations at an iteration $t$
such that \[
(1-\alpha)\norm{\ve_{1,t}}+\alpha\norm{\ve_{2,t}}\leq\frac{2\log(1+\tau)}{t}\sqrt{\frac{(n-p)(n'-p)}{N}}.\]
 which we simplify to performing the factorizations separately, and
stopping at the lowest iterations $t_{1}$ and $t_{2}$ such that
\begin{equation}
\begin{aligned}\norm{\ve_{1,t_{1}}} & \leq\frac{\log(1+\tau)}{(1-\alpha)t}\sqrt{\frac{(n-p)(n'-p)}{N}},\\
\norm{\ve_{2,t_{2}}} & \leq\frac{\log(1+\tau)}{\alpha t}\sqrt{\frac{(n-p)(n'-p)}{N}}.\end{aligned}
\label{eq:ve}\end{equation}
 We provide in Figure~\ref{fig:tradeoff} of Section \ref{subsec:results}
an experimental assessment of this speed/accuracy tradeoff when computing
the value of $\varphi_{\kappa}$.


\section{Experiments}

\label{sec:experiments} 
We provide in this section a fair assessment of the performance and
efficiency of autoregressive kernels on different tasks. We detail
in Section~\ref{subsec:kernels} the different kernels we consider
in this benchmark. Section~\ref{subsec:toy} and~\ref{subsec:real}
introduce the toy and real-life datasets of this benchmark, results
are presented in Section~\ref{subsec:results} before reaching the
conclusion of this paper.

\subsection{Kernels and parameter tuning}

\label{subsec:kernels} The kernels we consider in this experimental
section are all of the form $K=e^{-\frac{1}{t}\Phi}$, where $\Phi$
is a negative definite kernel. We select for each kernel $K$ the
value of the bandwidth $t$ as the median value $\hat{m}_{\Phi}$
of $\Phi$ on all pairs of time series observed in the training fold
times $0.5$, $1$ or $2$, namely $t\in\{.5\,\hat{m}_{\Phi},\,\hat{m}_{\Phi},2\,\hat{m}_{\Phi}\}$.
The selection is based on the cross validation error on the training
fold for each (kernel,dataset) pair. Some kernels described below
bear the superscript ${\cdot\,}^{\kappa}$, which means that they
are parameterized by a base kernel $\kappa$. Given two times series
$\bx=(x_{1},\cdots,x_{n})$ and $\bx'=(x_{1}',\cdots,x_{n'}')$, this
base kernel $\kappa$ is used to computed similarities between single
components $\kappa(x_{i},x'_{j})$ or p-uples of components $\kappa(\left(x_{i+1},\cdots,x_{i+p}\right),(x_{j+1}',\cdots,x_{j+p}'))$.
For all superscripted kernels below, $\kappa$ is set to be the Gaussian
kernel between two vectors $\kappa(x,y)=e^{-\norm{x-y}^{2}/(2\sigma^{2})}$,
where the dimension is obvious from the context and is either $d$
or $pd$. The variance parameter $\sigma^{2}$ is arbitrarily set
to be the median value of all Euclidean distances $\norm{x_{i}^{(r)}-x_{j}^{(s)}}$
where $i\leq|\bx^{(r)}|$, $j\leq|\bx^{(s)}|$, $(r,s)\in\Rcal^{2}$,
where $\Rcal$ is a random subset of $\{1,2,\cdots,\#\text{training points}\}$

\paragraph{Autoregressive kernels}

$k_{\ar},k_{\ar}^{\kappa}$: We consider the kernels \[
k_{\ar}=e^{-t_{\ar}\,\varphi},\quad k_{\ar}^{\kappa}=e^{-t_{\ar}^{\kappa}\,\varphi_{\kappa}},\]
 parameterized by the bandwidths $t_{\ar}$ and $t_{\ar}^{\kappa}$.
We set parameters $\alpha$ and $p$ as $\alpha=1/2$ and $p=5$ in
all experiments. The matrix factorizations used to compute approximations
to $k_{\ar}^{\kappa}$ (with $\tau$ set to $10^{-4}$) can be performed
using the \verb"chol_gauss" routine proposed by~\citet{shen2009fast}.
Since running separately this routine for $\bK_{1}$ and $\bK_{1}+\bK_{2}$
results in duplicate computations of portions of $\bK_{1}$, we have
added our modifications to this routine in order to cache values of
$\bK_{1}$ that can be reused when evaluating $\bK_{1}+\bK_{2}$.
The routine \verb"TwoCholGauss" is available on our website, as well as other pieces of code. We insist
on the fact that, other than $\alpha,p$ and the temperature $t_{\ar}$,
the autoregressive kernel $k_{\ar}$ does not require any parameter
tuning.

\paragraph{Bag of vectors kernel}

$\kBOV$: A time series $(x_{1},\cdots,x_{n})$ can be considered
as a bag of vectors $\{x_{1},\cdots,x_{n}\}$ where the time-dependent
information from each state's timestamp is deliberately ignored. Two
time series $\bx$ and $\bx'$ can be compared through their respective
bags $\{x_{1},\cdots,x_{n}\}$ and $\{x'_{1},\cdots,x'_{n'}\}$ by
setting \[
\psi_{\kappa}(\bx,\bx')\defeq\tfrac{1}{nn'}\sum_{i\leq n,j\leq n'}\kappa(x_{i},x'_{j}),\]
 and defining the kernel $\kBOV(\bx,\bx')=\exp\left(-t_{\text{BoV}}\left(\psi_{\kappa}(\bx,\bx)+\psi_{\kappa}(\bx',\bx')-2\psi_{\kappa}(\bx,\bx'\right)\right)$
where $t_{\text{BoV}}>0$, see for instance~\citep{hein05hilbertian}.
This relatively simple kernel will act as the baseline of our experiments,
both for performance and computational time.

\paragraph{Global alignment kernel}

$\kGA$: The global alignment kernel~\citep{cuturi07kernel} is a
positive definite kernel that builds upon the dynamic time warping
framework, by considering the soft-maximum of the alignment score
of all possible alignments between two time series. We use an implementation
of this kernel distributed on the web, and consider the kernel $\kGA=\exp(-t_{\text{GA}}\text{GlobalAlignment}_{\kappa}(\bx,\bx'))$,
parameterized by the bandwidth $t_{\text{GA}}$. Note that the global
alignment kernel has not been proved to be infinitely divisible. Namely,
$\kGA$ is known to be positive definite for $t_{\text{GA}}=1$, and
as a consequence for $t_{\text{GA}}\in\NN$, but not for all positive
values. However, the Gram matrices that were generated in these experiments
have been found to be positive definite for all values of $t_{\text{GA}}$
as discussed earlier in this section. This suggests that $\kGA$ might
indeed be infinitely divisible.

\paragraph{Splines Smoothing kernel}

$\kS$: ~\citet{kumara2008large} use spline smoothing techniques
to map each time series $(x_{1},\cdots,x_{n})$ onto a multivariate
polynomial function $p_{\bx}$ defined on $[0,1]$. As a pre-processing
step, each time series is mapped onto a multivariate time series of
arbitrary length $\tilde{\bx}$ (set to $200$ in our experiments)
such that $\tilde{\bx}^{T}\tilde{\bx'}$ corresponds to a relevant
dot-product for these polynomials. We have modified an implementation
that we received from the authors in email correspondence. ~\citet{kumara2008large}
consider a linear kernel in their original paper on such representations.
We have found that a Gaussian kernel between these two vector representations
performs better and use $\kS=\exp(-t_{\text{S}}\norm{\tilde{\bx}-\tilde{\bx'}}^{2})$.

\begin{rem}Although promising, the kernel proposed by~\citet[\S4.5]{jebara04probability}
is embryonic and leaves many open questions on its practical implementation.
A simple implementation using VAR models would not work with these
experiments, since for many datasets the dimension $d$ of the considered
time series is comparable to or larger than their lengths' and would
prevent any estimation of the $(pd^{2}+d(d+1)/2)$ parameters of a
$VAR(p)$ model. A more advanced implementation not detailed in the
original paper would be beyond the scope of this work. We have also
tried to implement the fairly complex families of kernels described
by~\citet{vishwanathan2007binet}, namely Equations (10) and (16)
in that reference, but our implementations performed very poorly on
the datasets we considered, and we hence decided not to report these
results out of concerns for the validity of our codes. Despite repeated
attempts, we could not obtain computer codes for these kernels from
the authors, a problem also reported in~\citep{lin2008alignment}.\end{rem}

\subsection{Toy dataset}

\label{subsec:toy} We study the performance of these kernels in a
simple binary classification toy experiment that illustrates some
of the merits of autoregressive kernels. We consider high dimensional
time series ($d=1000$) dimensional of short length ($n=10$) generated
randomly using one of two VAR(1) models, \[
x_{t+1}=A_{i}x_{t}+\varepsilon_{t},\quad i=1,2,\]
 where the process $\varepsilon_{t}$ is a white Gaussian noise with
covariance matrix $.1I_{1000}$. Each time series' initial point is
a random vector whose components are each distributed randomly following
the uniform distribution in $[-5,5]$. The two matrices $A_{i}$,
$i=1,2$, are sparse (10\% of non-zero values, that is 100.000 non
zero entries out of potentially one million) and have entries that
are randomly distributed following a standard Gaussian law%
\footnote{In Matlab notation, $A=\texttt{sprandn(1000,.1)}$%
}. These matrices are divided by their spectral radius to ensure that
their largest eigenvalue has norm smaller than one to ensure their
stationarity.

We draw randomly $10$ time series with transition matrix $A_{1}$
and $10$ times with transition matrix $A_{2}$ and use these $20$
time series as a training set to learn an SVM that can discriminate
between time series of type $1$ or $2$. We draw $100$ test time
series for each class $i=1,2$, that is a total of $200$ time series,
and test the performance of all kernels following the protocol outlined
above. The test error is represented in the leftmost bar plot of Figure~\ref{fig:final}.
The autoregressive kernel $k_{\ar}$ achieves a remarkable test error
of $0$, whereas other kernels, including $k_{\ar}^{\kappa}$, make
a not-so-surprisingly larger number of mistakes, given the difficulty
of this task. One of the strongest appeals of the autoregressive kernel
$k_{\ar}$ is that it manages to quantify a dynamic similarity between
two time series (something that neither the Kumara kernel or any kernel
based on alignments may achieve with so few samples) without resorting
to the actual estimation of a density, which would of course be impossible
given the samples' length.

\subsection{Real-life datasets}

\label{subsec:real} We assess the performance of the kernels proposed
in this paper using different benchmark datasets and other known kernels
for time series. The datasets are all taken from the UCI Machine Learning
repository~\citep{UCI}, except for the PEMS dataset which we have
compiled. The datasets characteristics' are summarized in Table~\ref{tab:dbs}.


\textbf{Japanese Vowels}: The database records utterances by nine
male speakers of two Japanese vowels `a' and `e' successively. Each
utterance is described as a time series of LPC cepstrum coefficients.
The length of each time series lies within a range of $7$ to $29$
observations, each observation being a vector of $\RR^{12}$. The
task is to guess which of the nine speakers pronounces a new utterance
of `a' or `e'. We use the original split proposed by the authors,
namely $270$ utterances for training and $370$ for testing.

\textbf{Libras Movement Data Set}: LIBRAS is the acronym for the brazilian
sign language. The observations are 2-dimensional time series of length
45. Each time series describes the location of the gravity center
of a hand's coordinates in the visual plane. 15 different signs are
considered, the training set has 24 instances of each class, for a
total $360=24\times15$ time series. We consider another dataset of
$585$ time series for the test set.

\textbf{Handwritten characters}: $2858$ recordings of a pen tip trajectory
were taken from the same writer. Each trajectory, a $3\times n$ matrix
where $n$ varies between $60$ and $182$ records the location of
the pen and its tip force. Each trajectory belongs to one out of $20$
different classes. The data is split into $2$ balanced folds of $600$
examples for training and $2258$ examples for testing.

\textbf{Australian Language of Signs}: Sensors are set on the two
hands of a native signer communicating with the AUSLAN sign language.
There are $11$ sensors on each hand and hence $22$ coordinates for
each observation of the time series. The length of each time series
ranges from $45$ to $136$ measurements. A sample of $27$ distinct
recordings is performed for each of the $95$ considered signs, which
totals $2565$ time series. These are split between balanced train
and tests sets of size $600$ and $1865$ respectively. Each time
series in both test and training sets is centered individually, that
is $\bx^{(i)}$ is replaced by $\bx^{(i)}-\bar{\bx}^{(i)}$. Without
such a centering the performance of all kernels is seriously degraded,
except for the autoregressive kernel which remains very competitive
with an error below 10\%.

\textbf{PEMS Database}: We have downloaded 15 months worth of daily
data from the California Department of Transportation PEMS website%
\footnote{\texttt{http://pems.dot.ca.gov}%
}. The data describes measurements at 10 minute intervals of occupancy
rate, between 0 and 1, of different car lanes of the San Francisco
bay area (D04) freeway system. The measurements cover the period from
January 1st 2008 to March 30th 2009. We consider each day of measurements
as a single time series of dimension $963$ (the number of sensors
which functioned consistently throughout the studied period) and length
$6\times24=144$. The task is to classify each day as the correct
day of the week, from Monday to Sunday, e.g. label it with an integer
between 1 and 7. We remove public holidays from the dataset, as well
as two days with anomalies (March 8 2009 and March 9 2008) where all
sensors have been seemingly turned off between 2:00 and 3:00 AM. This
leaves 440 time series in total, which are shuffled and split between
$267$ training observations and $173$ test observations. We plan
to donate this dataset to the UCI repository, and it should be available shortly. In the meantime, the dataset can be accessed in Matlab format on our website.

\begin{table}
\newcolumntype{U}{>{\centering\arraybackslash}X}
\begin{tabularx}{\textwidth}{|l|*{5}{U|}}
\hline 
	\textbf{Database}&$d$& $n$ &\textbf{classes}& \textbf{\# train } & \textbf{\#test}\\\hline\hline
	Toy dataset & 1000 & 10 & 2 & 20 & 200\\\hline\hline
	Japanese Vowels &12&7-29&9&270&370\\
	Libras &2&45&15&360&585\\
	Handwritten Characters&3&60-182&20&600&2258\\
	AUSLAN &22&45-136&95&600&1865\\
	PEMS &963&144&7&267&173 \\\hline	
\end{tabularx}
\caption{Characteristics of the different databases considered in the benchmark test}\label{tab:dbs}
\end{table}

\subsection{Results and computational speed}

\label{subsec:results} The kernels introduced in the Section above
are paired with a standard multiclass SVM implementation using a one-versus-rest
approach. For each kernel and training set pair, the SVM constant
$C$ to be used on the test set was chosen as either $1$, $10$ or
$100$, whichever gave the lowest cross-validation mean-error on the
training fold. We report the test errors in Figure~\ref{fig:final}.
The errors on the test sets can be also compared with the average
computation time per kernel evaluation graph displayed for 4 datasets
in Figure\ref{fig:tradeoff}. In terms of performance, the autoregressive
kernels perform favorably with respect to other kernels, notably the
Global Alignment kernel, which is usually very difficult to beat.
Their computational time offer a diametrically opposed perspective
since for these benchmarks datasets the flexibility of using a kernel
$\kappa$ to encode the inputs in a RKHS does not yield practical
gains in performance but has a tremendously high computational price.
On the contrary, $k_{\ar}$ is both efficient and usually very fast
compared to the other kernels.

\begin{figure}
\hskip-3.5cm \includegraphics[width=22cm]{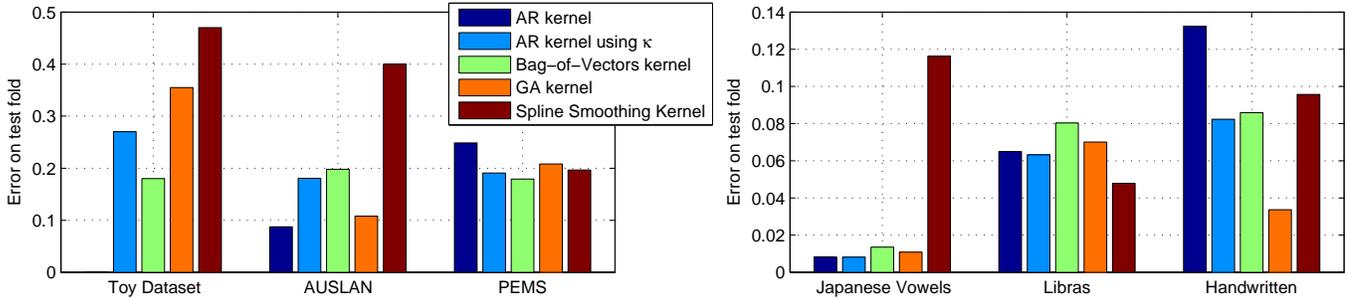} \caption{Test error of the 5 considered kernels on 5 different tasks split
into two panels for better legibility of the error rates (notice the
difference in scale). The AR kernel has a test-error of $0$ on the
toy dataset's test fold.\label{fig:final}}

\end{figure}

\begin{figure}
\hspace{-3cm}\includegraphics[width=21cm]{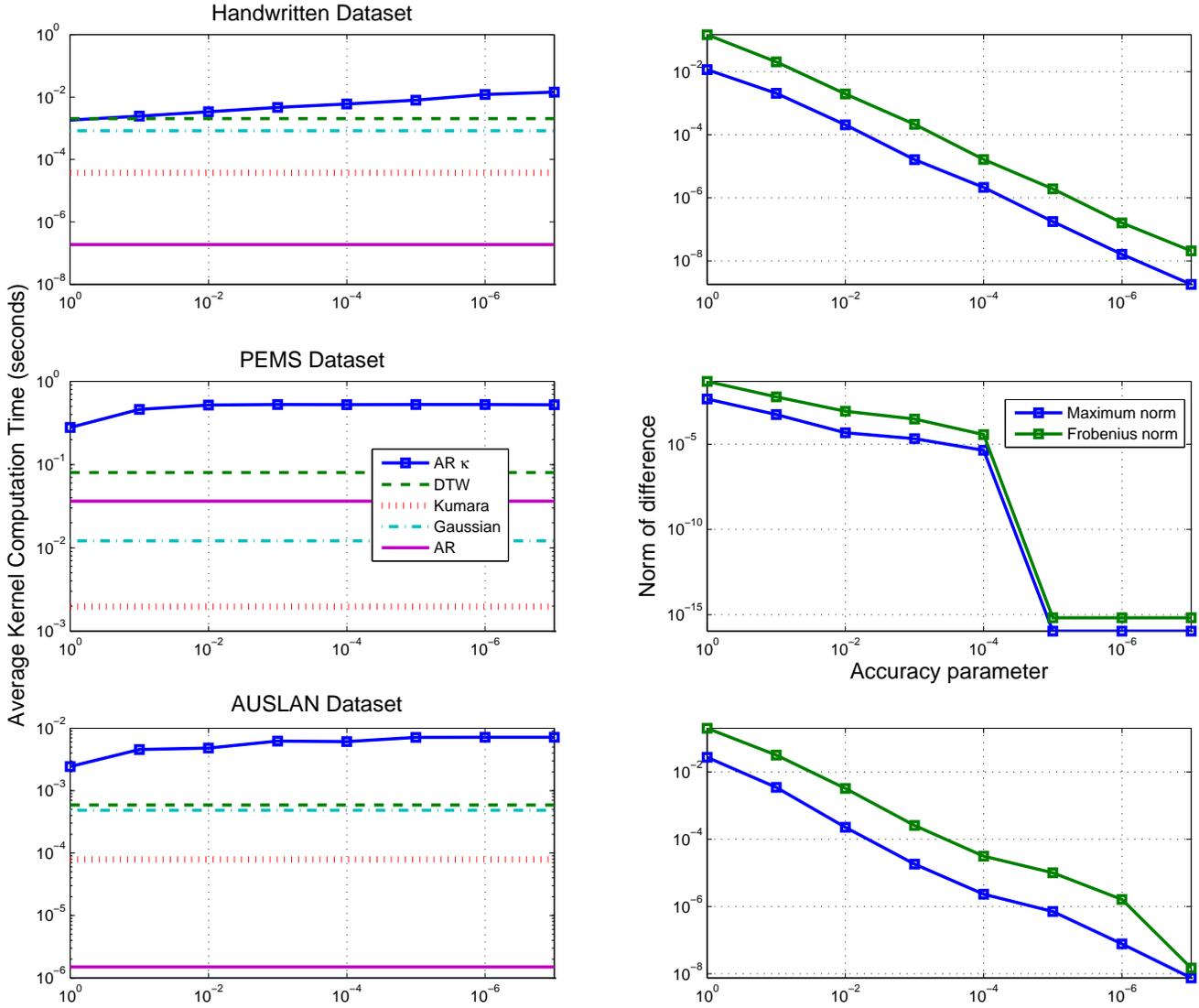} \caption{These graphs provide on the left side the average time needed to compute
one evaluation of each of the 5 kernels on the largest datasets. The
average speeds (computed over a sample of $50\times50$ calculations)
for each kernel are quantified in the y-axis. The x-axis is only effective
for the kernel $k_{\ar}^{\kappa}$ and shows the influence of the
accuracy para meter $\tau$ on that speed using the low-rank matrix
factorization expression used for $\varphi$ as described in Equation~\eqref{eq:ve}.
This parameter is set between $10^{0}$ (poor approximation) to $10^{-7}$
(high accuracy). The accuracy is measured by the maximum norm and
the Frobenius norm of the difference between the two $50\times50$
$\varphi$-Gram matrices. Note that $\kGA$ is fully implemented in
\texttt{mex}-C code, $k_{\ar}^{\kappa}$ uses \texttt{mex}-C subroutines
for Cholesky decomposition while all other kernels are implemented
using standard algebra in Matlab. Preprocessing times are not counted
in these averages. The simulations were run using an iMac 2.66 GhZ
Intel Core with 4Gb of memory.}

\label{fig:tradeoff} 
\end{figure}

\subsection{Conclusion and Discussion}\label{subsec:conclusion}

We have proposed in this work two infinitely divisible kernels $k_{\ar}$ and $k_{\ar}^{\kappa}$ for
time series. These kernels can be used within the framework of kernel
machines, \eg\, the SVM or kernel-PCA, or more generally as Hilbertian
distances by using directly their logarithms $\varphi_{\var}$ and $\varphi_{\var}^{\kappa}$ once properly normalized.
The first kernel $k_{\ar}$ computes a similarity between two multivariate
time series with a low computational cost. This similarity is easy
to implement, easy to tune given its infinite divisibility, and often performs similarly or better than more costly
alternatives. The second kernel, $k_{\ar}^{\kappa}$, is a generalization
of $k_{\ar}$ that can handle structured data by considering a local
kernel $\kappa$ on the structures. Its computation requires the computation
of all or a part of large Gram matrices as well as the determinant
of these. Given its computational drawbacks, the experimental evidence
gathered in this paper is not sufficient to advocate its use on vectorial
data. Moreover,~\citet{el2010spectrum} has recently shown that the
spectrum of a Gram matrix of high-dimensional points using the Gaussian
kernel may, under certain assumptions, be very similar to the spectrum
of the standard Gram matrix of these same points using the linear
dot-product. In such a case, the sophistication brought forward by
$k_{\ar}^{\kappa}$ might be gratuitous and yield similar results
to the direct use of $k_{\ar}$. However, we believe that $k_{\ar}^{\kappa}$
may prove particularly useful when considering time series of structured
data. For instance, we plan to apply the kernel $k_{\ar}^{\kappa}$
to the classification of video segments, where each segment would
be represented as a time varying histogram of features and $\kappa$
a suitable kernel on histograms that can take into account the similarity
of features themselves. 

Our contribution follows the blueprint laid down by~\cite{seeger02covariance} which can be effectively applied to other exponential models. However, we believe that the infinite divisibility of such kernels, which is crucial in practical applications, had not been considered before this work. Our result in this respect is not as general as we would wish for, in the sense that we do not know whether $k_{\ar}$ remains infinitely divisible when
the degrees of freedom $\lambda$ of the inverse Wishart prior exceed
$d-1$. In such a case, the concavity of $f$ in Equation~\eqref{eq:definitionsFandG} would not be given either. 
Finally, although the prior that we use to define $k_{\ar}$ is non informative, it might be of interest
to learn the hyperparameters for these priors based on a data corpus of interest.

\bibliographystyle{abbrvnat} \bibliographystyle{abbrvnat}
\bibliography{bib_short}
 
\end{document}